\documentclass{ieeeaccess}

\providecommand{\xfigwd}{0pt}
\usepackage{cite}
\usepackage{amsmath,amssymb,amsfonts}
\usepackage{mathtools}
\usepackage{amsthm}
\usepackage{algorithm}
\usepackage{algpseudocode}
\usepackage{graphicx}
\usepackage{subfigure}
\usepackage{booktabs}
\usepackage{placeins}
\usepackage{microtype}
\usepackage[hidelinks]{hyperref}
\usepackage[capitalize,noabbrev]{cleveref}
\usepackage{textcomp}

\def\BibTeX{{\rm B\kern-.05em{\sc i\kern-.025em b}\kern-.08em
    T\kern-.1667em\lower.7ex\hbox{E}\kern-.125emX}}
 
\theoremstyle{plain}
\newtheorem{theorem}{Theorem}[section]
\newtheorem{proposition}[theorem]{Proposition}

\theoremstyle{definition}

\theoremstyle{remark}
\newtheorem{remark}[theorem]{Remark}

\begin{document}
\history{Preprint. This work has been submitted for peer review.}
\doi{}

\title{Similarity Weighted Aggregation with Global Differential Privacy for Federated Brain Lesion Segmentation}

\author{\uppercase{Muhammad Irfan Khan}\authorrefmark{1},
\uppercase{Eero Lehtonen}\authorrefmark{1},
\uppercase{Joni Obradovic}\authorrefmark{1},
\uppercase{Elina Kontio}\authorrefmark{1},
\uppercase{Esa Alhoniemi}\authorrefmark{1},
\uppercase{Suleiman A. Khan}\authorrefmark{1},
\uppercase{and Mojtaba Jafaritadi}\authorrefmark{1}}

\address[1]{Turku University of Applied Sciences, Turku 20520, Finland (e-mail: irfan.khan@turkuamk.fi, eero.lehtonen@turkuamk.fi, 
joni.p.obradovic@turkuamk.fi,
elina.kontio@turkuamk.fi, esa.alhoniemi@turkuamk.fi, suleiman.alikhan@turkuamk.fi, mojtaba.jafaritadi@turkuamk.fi)}

\tfootnote{This work was carried out as part of a research project at Turku University of Applied Sciences. This work was supported by Business Finland under Grant 1337/31/2024.}

\markboth
{IEEE Access}
{Khan \headeretal: DP-SimAgg for Privacy-Preserving Federated Brain Lesion Segmentation}

\corresp{Corresponding author: Muhammad Irfan Khan (e-mail: irfan.khan@turkuamk.fi).}

\begin{abstract}
Federated Learning (FL) enables collaborative training of machine learning models across multiple institutions without sharing sensitive data, making it particularly suitable for medical imaging applications. However, heterogeneous data distributions across institutions and potential information leakage through model updates remain important challenges. In this work, we propose DP-SimAgg, a privacy-preserving federated learning framework that integrates similarity-weighted aggregation with a server-side differential privacy mechanism. The proposed method applies L2 clipping to bound collaborator updates, computes similarity-based aggregation weights to mitigate the effects of non-IID data distributions, and injects calibrated Gaussian noise at the central server and provides per-round privacy guarantees under the assumed sensitivity bound. 
The framework is implemented using Intel's OpenFL platform and evaluated on the FeTS 2022 dataset consisting of 1251 multi-modal MRI scans for brain tumor segmentation. Experimental results demonstrate that DP-SimAgg maintains competitive segmentation performance while providing privacy protection. For instance, under a strict per-round privacy budget ($\epsilon = 1$, cumulative $\epsilon_{\text{total}} = 20$ over 20 rounds), the method achieves Dice scores of 0.6357, 0.5305, and 0.5274 for the enhancing tumor (ET), tumor core (TC), and whole tumor (WT) regions, respectively. With a more relaxed per-round budget ($\epsilon = 10$, cumulative $\epsilon_{\text{total}} = 200$), performance approaches that of the non-private baseline 
while incorporating a central Gaussian mechanism with per-round $(\epsilon,\delta)$-DP accounting under the assumed sensitivity bound.
These results highlight the potential of DP-SimAgg for enabling privacy-preserving collaborative learning in medical imaging applications.
\end{abstract}


\titlepgskip=-15pt
\maketitle

\section{Introduction}
\label{sec:introduction}

Collaborative machine learning across medical institutions holds great promise for improving diagnostic accuracy and clinical decision support, yet it is fundamentally constrained by two intertwined challenges: data scarcity and privacy regulation. In medical imaging, developing high-performance deep learning models requires large, heterogeneous, and carefully annotated datasets~\cite{varoquaux2022machine}. Brain tumor segmentation is a clinically critical example: early and precise delineation of glioblastoma (GBM) sub-regions from multi-parametric MRI is strongly associated with treatment planning quality and patient outcomes~\cite{sheller2020federated,bakas_advancing_2017}. 
Yet patient data are tightly protected under legislation such as HIPAA and GDPR~\cite{may2010hipaa2, tinja}, making centralized data pooling across institutions legally restricted or outright prohibited. Within a single institution, annotated datasets are often too small and insufficiently diverse to train generalizable models, especially for rare pathologies~\cite{ng2021federated}.
 
Federated learning (FL) was introduced precisely to reconcile this tension~\cite{mcmahan2017communication}. In the standard FL protocol, a central server distributes a global model to a set of collaborating institutions, each of which trains locally on its private data and returns only model parameter updates to the server. The server aggregates these updates, most commonly via weighted averaging, and broadcasts the refined global model for the next round. At no point is raw patient data transmitted, offering a principled alternative to centralized training under privacy constraints~\cite{wang_yurochkin_papailiopoulos_khazaeni_2020, khan_regsimagg23}.
 
Despite this, FL in its standard form faces two persistent challenges that are particularly severe in the medical imaging domain. The first is \emph{statistical heterogeneity}: patient populations, imaging protocols, scanner hardware, and annotation conventions differ substantially across institutions, producing non-independent and non-identically distributed (non-IID) local datasets~\cite{zhao2018federated, Kairouz2019}. When local data distributions diverge significantly, standard federated averaging (FedAvg) is prone to \emph{client drift}~\cite{karimireddy2020scaffold}, a phenomenon where locally optimized updates pull the global model in conflicting directions, slowing convergence and degrading final model quality. The second challenge is \emph{privacy leakage}: transmitting model gradients or parameters can expose sensitive information about the underlying training data. Gradient inversion attacks~\cite{https://doi.org/10.48550/arxiv.2003.14053} and membership inference attacks~\cite{DBLP:journals/corr/ShokriSS16} have demonstrated that an adversarial server or eavesdropper can recover private patient attributes or even approximate training samples from shared model updates alone. Vanilla FL thus provides no formal privacy guarantee~\cite{truong_privacy_2021, rodriguez2023survey, https://doi.org/10.48550/arxiv.2003.14053}.

Addressing both challenges simultaneously requires mechanisms that are sensitive to update quality across heterogeneous collaborators while also providing mathematically certified privacy guarantees. In this work, we propose \textbf{DP-SimAgg}, a federated learning framework that directly addresses these two challenges through a tightly integrated design: a \emph{similarity-weighted aggregation} (SimAgg)~\cite{khan2021adaptive}  scheme that down-weights divergent collaborator updates to mitigate client drift under non-IID conditions, combined with a \emph{server-side Gaussian differential privacy} mechanism that provides formal $(\epsilon, \delta)$-DP guarantees under an empirically estimated sensitivity per communication round. The key insight is that these two components are empirically observed to be mutually reinforcing: L2 norm clipping, required for bounding the DP sensitivity, simultaneously prevents outlier updates from corrupting the global model, and the similarity weights are computed \emph{after} clipping, ensuring that the DP guarantee is not undermined by the weighting scheme. Our prior work demonstrated that SimAgg  significantly outperforms FedAvg and related methods in non-IID medical imaging settings, including winning multiple benchmarking challenges ~\cite{zenk2025towards,Linardos_2025}. However, its effectiveness under differential privacy constraints has not been studied. This work fills that gap by systematically evaluating the robustness of similarity-weighted aggregation under varying privacy budgets, and by integrating it with a global differential privacy mechanism to analyze the trade-off between model utility and privacy in federated brain tumor segmentation.
 
The framework is implemented within Intel's OpenFL platform~\cite{reina_openfl_2022} and evaluated on the FeTS~2022 federated brain tumor segmentation benchmark~\cite{zenk2025towards}, comprising 1251 multi-institutional multi-parametric MRI scans across 33 collaborating sites. Experiments demonstrate that DP-SimAgg achieves competitive segmentation performance across a range of privacy budgets, with near-baseline Dice scores at a per-round budget of $\epsilon_0 = 10$ (cumulative $\epsilon_{\text{total}} = 200$ over 20 rounds) and meaningful segmentation quality even under strict privacy ($\epsilon_0 = 1$, $\epsilon_{\text{total}} = 20$), at computational cost comparable to non-private training.
 
The specific contributions of this paper are as follows:
\begin{itemize}
    \item We propose \textbf{DP-SimAgg}, a federated aggregation framework that unifies similarity-weighted update selection with central Gaussian differential privacy for multi-institutional medical image segmentation.
    \item We design a server-side, empirically sensitivity-aware privacy mechanism combining per-collaborator L2 norm clipping, similarity-based aggregation weighting, empirical L2 sensitivity estimation via neighboring-dataset simulation, and calibrated Gaussian noise injection --- all applied without modifying collaborator-side training.
    \item We provide complete algorithmic specifications (Algorithms~1--3) with formal $(\epsilon_0, \delta_0)$-DP guarantees per round and explicit sequential composition accounting for the multi-round training setting.
    \item We conduct an empirical evaluation on the FeTS~2022 benchmark, quantifying the privacy--utility trade-off across privacy budgets and demonstrating computational parity with non-private federated training.
\end{itemize}
 
The remainder of this paper is organized as follows. Section~\ref{sec:relatedwork} surveys related work on FL aggregation, differential privacy in FL, and privacy-preserving medical imaging. Section~\ref{sec:method} describes the proposed framework and algorithmic design. Section~\ref{sec:results} presents the experimental evaluation, and Section~\ref{sec:discussion} discusses the results, limitations, and future directions.
 
\section{Related Work}
\label{sec:relatedwork}
 
\subsection*{Federated Learning Aggregation and Heterogeneity}
 
The canonical FL aggregation algorithm, FedAvg~\cite{mcmahan2017communication}, computes a weighted average of collaborator updates proportional to local dataset size. While effective under IID conditions, FedAvg degrades under non-IID data distributions due to client drift, where heterogeneous local optima cause diverging update directions~\cite{zhao2018federated, Kairouz2019}. Numerous aggregation strategies have been proposed to address this. FedProx~\cite{li2020federated} adds a proximal regularization term to each collaborator's local objective to limit divergence from the global model. SCAFFOLD~\cite{karimireddy2020scaffold} uses control variates to correct for client drift directly in the gradient update. FedNova~\cite{wang2020tackling} normalizes local updates by the number of local steps to eliminate objective inconsistency. These methods, however, operate on the \emph{collaborator} side and require modifications to local training procedures, which may be impractical in heterogeneous clinical environments where institutions use different software stacks or lack the infrastructure for custom training loops.
 
Server-side aggregation strategies offer a more deployment-friendly alternative. 
Krum and Multi-Krum~\cite{blanchard2017machine} select updates that are closest to other received updates, providing robustness against Byzantine collaborators. FedMedian~\cite{yin2018byzantine} uses coordinate-wise median aggregation for similar effect. Similarity-based aggregation, as developed in our prior work ~\cite{khan_regsimagg23} and ~\cite{khanadaptive}, weights collaborator contributions by their distance from the global consensus, assigning higher influence to updates that are consistent with the aggregate direction. This approach is purely server-side and requires no modification to local training, making it well-suited for real-world federated medical imaging deployments. The present work extends this line to incorporate formal differential privacy guarantees.
 
\subsection*{Differential Privacy in Federated Learning}
 
Differential privacy (DP) provides a rigorous mathematical framework for limiting information leakage from published model outputs~\cite{10.1007/11787006_1, Dwork2013, Dwork2011}. A mechanism $\mathcal{M}$ satisfies $(\epsilon, \delta)$-DP if its output distribution changes by at most a multiplicative factor $e^\epsilon$ (plus additive slack $\delta$) when any single individual's data is added or removed from the input. The Gaussian mechanism~\cite{abadi2016deep} achieves this by adding isotropic Gaussian noise calibrated to the L2 sensitivity of the target function.
 
The integration of DP with FL was formalized by McMahan et al.~\cite{abadi2016deep} in DP-FedAvg, which clips per-sample gradients and adds Gaussian noise at the server, using the moments accountant for tight privacy composition. However, DP-FedAvg was designed for the local DP setting and applies clipping and noise at the \emph{collaborator} level during local training, which requires access to per-sample gradients and imposes significant computational overhead. Geyer et al.~\cite{geyer2017differentially} proposed an alternative user-level DP scheme that clips and noises full model updates at the server rather than per-sample gradients, reducing collaborator-side complexity. Subsequent work has explored local DP~\cite{DBLP:journals/corr/abs-1905-02383}, where collaborators independently perturb their updates before transmission, at the cost of substantially higher noise and reduced utility compared to central DP for the same privacy budget. In this work, we adopt the \emph{central} (server-side) DP model, consistent with Geyer et al.~\cite{geyer2017differentially}, which provides stronger utility guarantees under the standard trusted-aggregator assumption.
 
Privacy accounting for multi-round FL is a non-trivial problem. Naive 
sequential composition yields a linear budget growth of $T\epsilon_0$ over $T$ rounds. The moments accountant~\cite{abadi2016deep} and R\'{e}nyi DP~\cite{mironov2017renyi} provide substantially tighter bounds. In this work we report the basic sequential composition bound for transparency, and note that tighter accounting via these methods is a straightforward improvement.
 
\subsection*{Privacy-Preserving Federated Medical Imaging}
 
The application of FL to medical image segmentation has grown substantially since Sheller et al.~\cite{sheller2020federated} demonstrated FL for brain tumor segmentation without centralizing data. Subsequent work has addressed non-IID challenges in multi-site medical imaging~\cite{Linardos_2025}, collaborator selection strategies~\cite{khan2024electioncollaboratorsreinforcementlearning, khan2024recommenderenginedrivenclient}, and the FeTS benchmark for standardized federated evaluation~\cite{zenk2025towards}.
 
The intersection of DP and medical FL has received increasing attention. Kaissis et al.~\cite{kaissis2020} demonstrated DP-FL for chest radiograph classification and highlighted the sensitivity of medical imaging models to noise injection. Bernstein et al.~\cite{bernstein2018signsgd} and Jin et al.  ~\cite{jin2025noisy} studied signSDG and DP for imaging datasets, noting that high-dimensional inputs complicate sensitivity estimation. Critically, neither work addresses the interaction between aggregation strategy and DP: they apply noise to standard FedAvg aggregates, leaving the non-IID problem unaddressed. Roth et al.~\cite{eisenmann2023biomedicalimageanalysiscompetitions} reviewed privacy challenges in biomedical image analysis competitions, noting that gradient-based attacks remain viable against standard FL in medical settings. Our work differs from all of the above in that it jointly optimizes the aggregation scheme for non-IID robustness and applies calibrated DP noise, treating both problems as inseparable within a single server-side mechanism.
 
\subsection*{Positioning of This Work}
 
Table~\ref{tab:related} summarizes the key design choices of the most closely related methods. Relative to prior work, DP-SimAgg is, to our knowledge, the first federated aggregation framework to explicitly combine \emph{server-side} similarity-weighted aggregation with \emph{server-side} central Gaussian DP in the context of multi-institutional medical image segmentation, without requiring any modification to collaborator-side training procedures.
 
\begin{table}[t]
\caption{Comparison of related federated learning methods along key design axes. Server-side only indicates no modification to collaborator training is required. Non-IID robust indicates explicit mechanism for heterogeneous data. Formal DP indicates $(\epsilon,\delta)$-DP guarantee.}
\label{tab:related}
\centering
\resizebox{\columnwidth}{!}{
\begin{tabular}{lccc}
\toprule
Method & Server-side only & Non-IID robust & Formal DP \\
\midrule
FedAvg~\cite{mcmahan2017communication}         & \checkmark & $\times$   & $\times$ \\
FedProx~\cite{li2020federated}                 & $\times$   & \checkmark & $\times$ \\
SCAFFOLD~\cite{karimireddy2020scaffold}        & $\times$   & \checkmark & $\times$ \\
DP-FedAvg~\cite{abadi2016deep}                 & $\times$   & $\times$   & \checkmark \\
Geyer et al.~\cite{geyer2017differentially} & \checkmark & $\times$   & \checkmark \\
SimAgg~\cite{khan_regsimagg23}                 & \checkmark & \checkmark & $\times$ \\
\textbf{DP-SimAgg (ours)}                      & \checkmark & \checkmark & \checkmark \\
\bottomrule
\end{tabular}
}
\end{table}

\section{Methods}
\label{sec:method}
\subsection{Dataset Overview}
All experiments are conducted using the training split of the FeTS~2022 federated segmentation challenge~\cite{zenk2025towards}, specifically Partitioning~2, which distributes 1251 patients with confirmed glioblastoma (GBM) diagnoses across 33 institutions. This partitioning was designed to reflect realistic inter-institutional data heterogeneity, with institutions contributing between 4 and 171 patients each (see Fig.~\ref{fig:patientdistr}), making it a suitable benchmark for evaluating federated learning under non-IID conditions.

\begin{figure*}[th]
    \centering
    \includegraphics[width=\textwidth, angle=0]{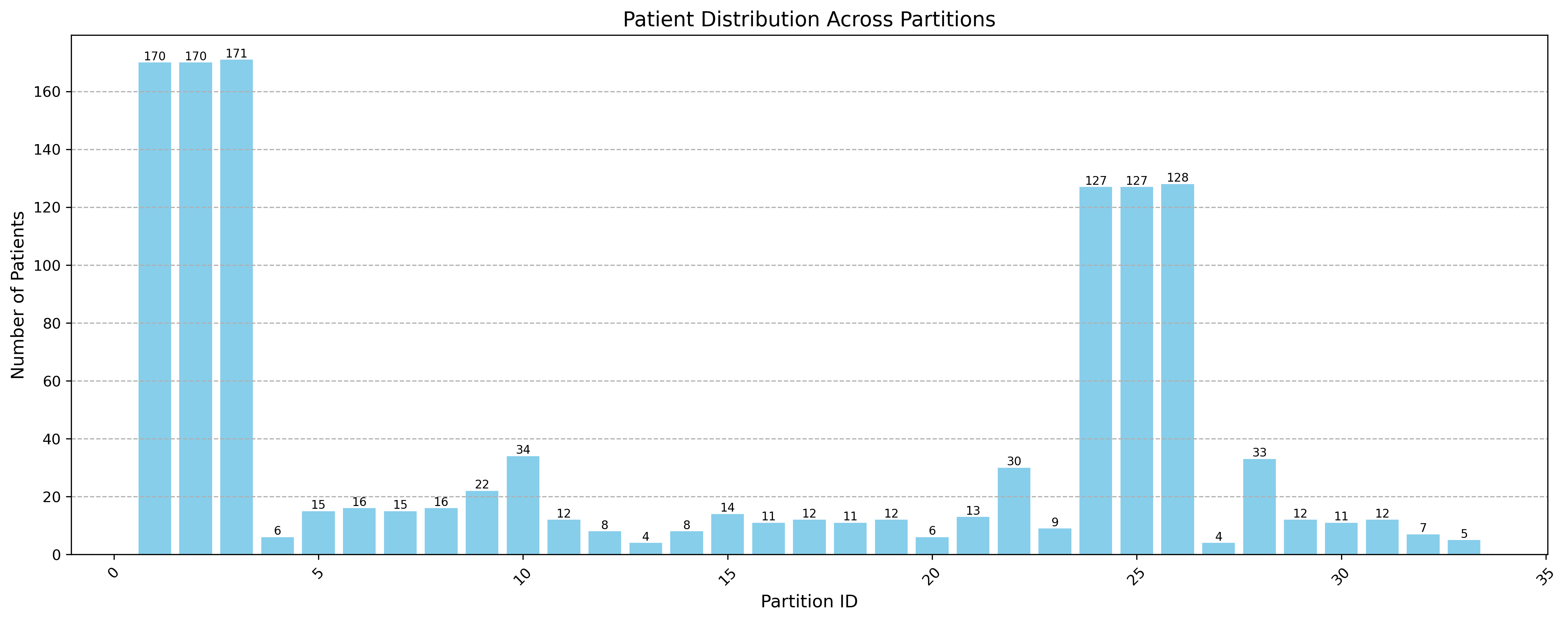}
    \caption{Official partitioning of the training set and patient distribution per institution.}
    \label{fig:patientdistr}
\end{figure*}
 
Glioblastoma multiforme is classified as a WHO Grade~IV primary brain tumour arising from astrocytes~\cite{kleihues1993new}. It is characterised by rapid infiltrative growth, typically localised to the frontal or temporal lobes, and carries a median survival of 14--16 months even with standard treatment~\cite{hanif2017glioblastoma}. Precise sub-region delineation from pre-operative MRI is therefore clinically critical for surgical planning, radiotherapy targeting, and treatment monitoring.
 
Each patient's imaging data consists of four co-registered, skull-stripped structural MRI volumes of size $240 \times 240 \times 155$ voxels at 1\,mm isotropic resolution, stored in NIfTI format (\texttt{.nii.gz}): native T1-weighted (T1), contrast-enhanced T1-weighted (T1Gd), T2-weighted (T2), and T2 Fluid-Attenuated Inversion Recovery (FLAIR). Representative examples of each modality and their corresponding tumor sub-regions are shown in Fig.~\ref{fig:seg}. Expert annotations provide three tumor sub-region labels following the BraTS convention~\cite{bakas_advancing_2017}: the necrotic tumor core (NCR, label~1), the peritumoral oedematous tissue (ED, label~2), and the GD-enhancing tumor (ET, label~4). Three clinically relevant composite regions are defined from these labels for evaluation: the enhancing tumor (ET, label~4), the tumor core (TC~$=$~NCR~$\cup$~ET, labels~1\,+\,4), and the whole tumor (WT~$=$~NCR~$\cup$~ED~$\cup$~ET, labels~1\,+\,2\,+\,4).
 
All imaging volumes were pre-processed centrally prior to federated distribution. The pipeline includes rigid co-registration of all modalities to the T1Gd volume, affine registration to the SRI24 anatomical atlas, resampling to 1\,mm$^3$ isotropic resolution, and automated skull stripping, following the procedure detailed in~\cite{bakas_advancing_2017,bakas_segmentation_2017,bakas_segmentation_2017-1}.
 
\begin{figure}[h]
    \centering
    \includegraphics[width=\columnwidth, angle=0]{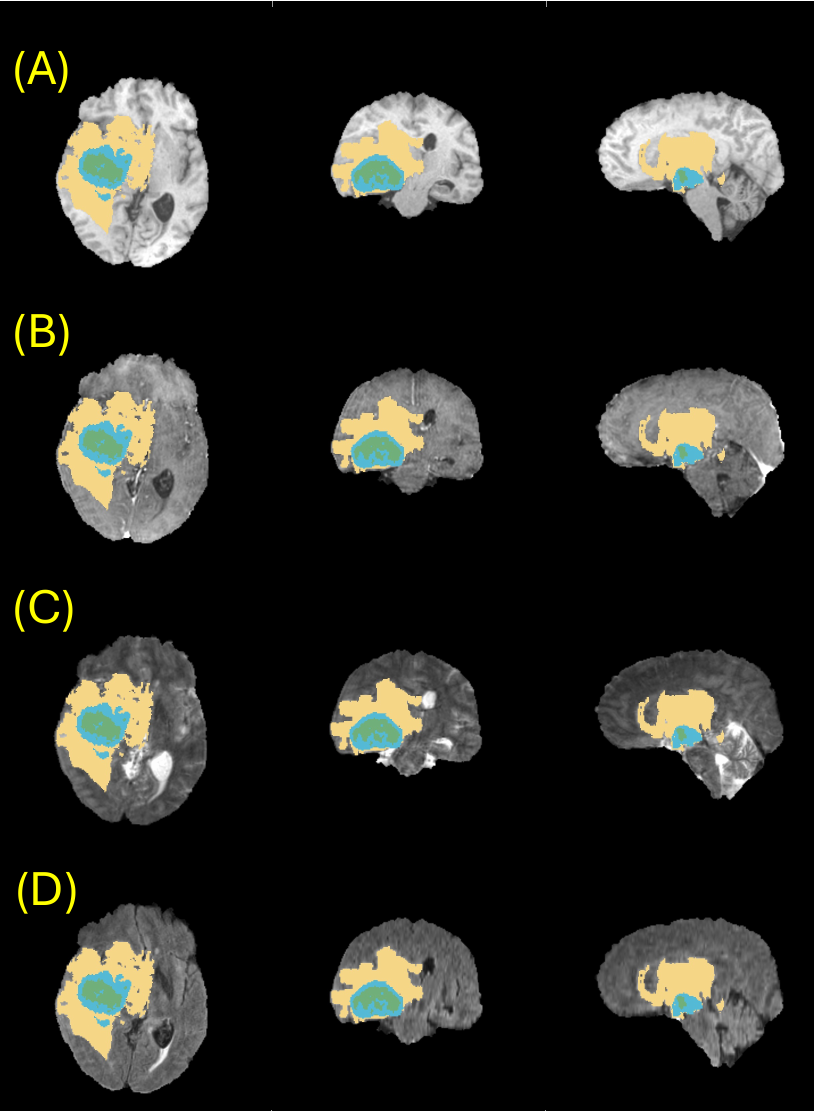}
    \caption{Representative glioblastoma (GBM) lesions across four mpMRI modalities: native T1 (A), contrast-enhanced T1Gd (B), T2 (C), and FLAIR (D). Tumor sub-regions are colour-coded: whole tumor including oedema (yellow), GD-enhancing tumor (blue), and necrotic core (green).}
    \label{fig:seg}
\end{figure}
 
\subsection{Federated Learning Framework}
 
Federated training is conducted using Intel's OpenFL framework~\cite{reina_openfl_2022}, an open-source FL platform designed for secure, distributed model training across institutional boundaries without raw data exchange. OpenFL organises training around two roles: \emph{collaborators}, each holding a private local dataset and executing model training on their own hardware; and an \emph{aggregator}, which coordinates the training schedule, collects parameter updates from active collaborators, and produces an updated global model at the end of each communication round. No patient data or raw images are transmitted at any stage; instead, only model tensors are exchanged.
 
The segmentation model is a residual U-Net~\cite{ronneberger_u-net_2015} with 95 layers and approximately 33 million parameters, provided by the FeTS~2022 challenge organisers~\cite{spyridon_bakas_2022_6362409} as the standard benchmark architecture for this task. The network follows an encoder--decoder design with skip connections that preserve spatial resolution across scales, which is particularly important for delineating small sub-regions such as the necrotic core.
 
Training is conducted over 20 communication rounds. At each round, a subset of collaborators trains the current global model locally on their private data for a fixed number of local epochs, then transmits their updated model tensors to the aggregator. The aggregator applies the DP-SimAgg procedure (Section~\ref{sec:method}) to produce the next global model, which is then broadcast to all collaborators for the subsequent round. All experiments are executed on a cluster node equipped with an NVIDIA Tesla V100 GPU and 350\,GB of system memory.

\subsection{Collaborator Selection Strategy}

Let $K$ denote the total number of available collaborators and $\rho \in (0,1]$ the participation ratio. In each communication round $t$, a subset $\mathcal{S}_t$ of $|\mathcal{S}_t| = \lfloor \rho K \rfloor$ collaborators is selected to participate in training. In our experiments, $K = 33$ and $\rho = 0.2$, yielding $|\mathcal{S}_t| = 6$ active collaborators per round. The selected set $\mathcal{S}_t$ is passed to Algorithm~1 as the \texttt{collaborators\_chosen} parameter, which determines the local tensors aggregated in that round.

Full participation ($\rho = 1$) is impractical in realistic federated deployments due to system heterogeneity, communication constraints, and variable device availability. However, naive 
independent random sampling at each round can result in uneven participation frequencies: some collaborators may be selected far more often than others over the course of training, introducing a systematic bias in the aggregated model toward over-represented data distributions.

To address this, we use a \emph{sliding window selection mechanism} over a randomly permuted collaborator list. At the beginning of each cycle, a uniformly random permutation $\pi$ of all $K$ collaborators is generated. Across consecutive rounds, a non-overlapping window of $|\mathcal{S}_t|$ collaborators is read sequentially from $\pi$. Once the window reaches the end of the permuted list (after $\lceil K / |\mathcal{S}_t| \rceil$ rounds), a new random permutation is generated and the process repeats. This guarantees that each collaborator appears \emph{approximately once} per cycle, with at most a few collaborators appearing in two consecutive windows at cycle boundaries due to integer rounding. Over the 20 training rounds used in our experiments, this yields 4 cycles, ensuring balanced participation across all institutions.

As illustrated in Fig.~\ref{fig:colab}, this approach differs from: (a)~\emph{static selection}, which fixes the same subset across all rounds and risks systematic overfitting to a data subset; and (b)~\emph{randomized order}, which visits all collaborators in a fixed random sequence without refreshing the permutation, limiting stochasticity across cycles. The proposed sliding window combines the coverage guarantee of ordered selection with the stochasticity of fresh permutations each cycle.

This mechanism provides two key properties:
\begin{itemize}
    \item \textbf{Fairness:} Each collaborator participates approximately once per cycle, preventing any institution from dominating or being systematically excluded from training.
    \item \textbf{Diversity:} Exposing the model to all data distributions in a structured, stochastic order improves generalization under non-IID conditions.
\end{itemize}

\begin{figure}[h]
    \centering
    \includegraphics[width=\columnwidth, angle=0]{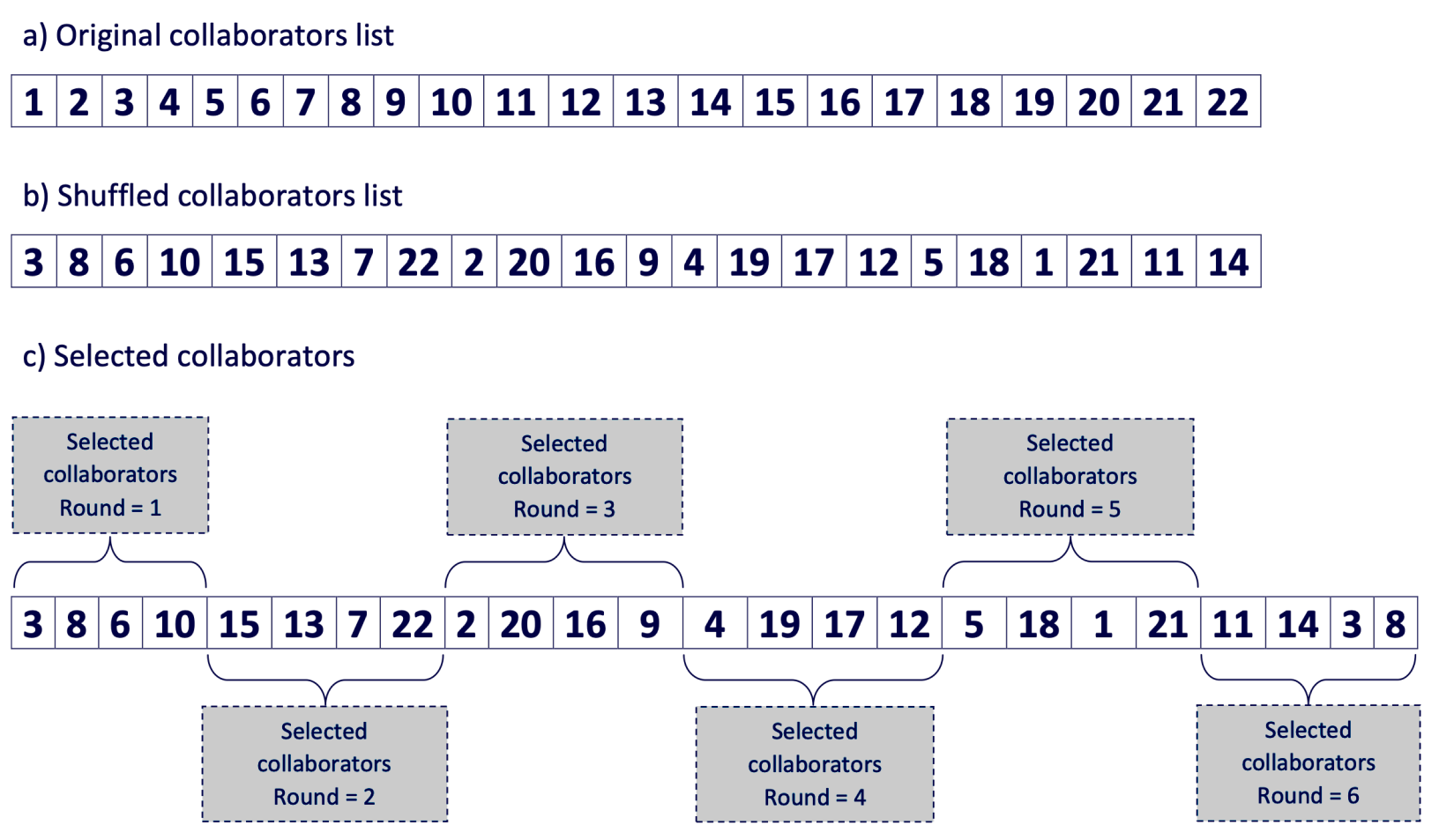}
    \caption{Illustration of the collaborator selection strategies: (a) Static Selection, (b) Randomized Order, and (c) Sliding Window approach, which ensures diversity and fairness in collaborator participation during federated learning.
    }
    \label{fig:colab}
\end{figure}

\subsection{Differential Privacy-Enabled Adaptive Weight Aggregation}

The proposed DP-SimAgg aggregation procedure combines similarity-aware weighting with a global differential privacy mechanism to address both statistical heterogeneity and privacy risks in federated learning environments. In heterogeneous medical datasets, collaborator updates may vary significantly due to differences in imaging protocols, patient populations, and institutional data distributions. To ensure robust aggregation while maintaining privacy guarantees, the proposed method incorporates four key stages: (i) bounding collaborator updates using L2 clipping, (ii) computing similarity-based aggregation weights, (iii) estimating aggregation sensitivity through a simulated neighboring dataset, and (iv) injecting calibrated Gaussian noise according to differential privacy parameters.

Prior to aggregation, each collaborator's model update $v_i$ is projected onto an L2 ball of radius $C$ via $\tilde{v}_i = v_i \cdot \min\!\left(1,\, C/\|v_i\|_2\right)$, ensuring $\|\tilde{v}_i\|_2 \leq C$ for all $i$. This bounds the influence of any single collaborator on the global model. From a differential privacy perspective, clipping is essential because it establishes a finite upper bound on the L2 sensitivity of the aggregation function. Without this norm constraint, the difference between aggregation outputs on neighboring datasets could grow arbitrarily large, requiring excessive noise injection that would degrade model utility.

After clipping, a similarity-weighted aggregation mechanism is applied. Unlike standard federated averaging, which typically assigns average weights of participating collaborators, 
the proposed method considers the similarity between each collaborator update and the global average update. The distance between local updates and the averaged global estimate is computed, and collaborators whose updates are closer to the consensus receive larger weights. This strategy improves robustness in non-IID federated learning environments by reducing the influence of highly divergent updates that may arise from skewed data distributions across institutions.

To estimate the sensitivity $\Delta f$, we simulate a neighboring dataset by replacing one collaborator update with a bounded adversarial tensor constrained by the clipping threshold $C$. This procedure provides an empirical approximation of the aggregation perturbation under worst-case bounded deviations. Although this approach does not yield a formal tight upper bound on the sensitivity, it offers a practical and tractable estimate for calibrating the Gaussian noise in our setting.

Finally, Gaussian noise is injected into the aggregated tensor using the estimated sensitivity $\Delta f$ and the per-round privacy budget parameters $(\epsilon_0, \delta_0)$, ensuring that the aggregation satisfies $(\epsilon_0,\delta_0)$-differential privacy for that round.

\subsection{Global Differential Privacy (GDP) in Federated Lesion Segmentation}

Global Differential Privacy (GDP) in federated learning protects the privacy of individual participants by introducing calibrated noise into the aggregated model updates at the central server. This centralized approach provides a practical balance between data utility and privacy preservation while avoiding the excessive noise accumulation often associated with local differential privacy mechanisms. By applying privacy protection only at the aggregation stage, GDP allows collaborators to train models locally without modifying their training procedures, while still safeguarding sensitive information.

The privacy guarantee follows the standard definition of $(\epsilon,\delta)$-differential privacy~\cite{10.1007/11787006_1,Dwork2013}. In the federated setting, two dataset profiles $D_1$ and $D_2$ are considered \emph{neighboring} if they differ in the contribution of exactly one collaborator (i.e., one collaborator's local model update is replaced or removed). A randomized mechanism $\mathcal{M}$ satisfies $(\epsilon,\delta)$-DP if, for all such neighboring profiles and all measurable output sets $S \subseteq \mathrm{Range}(\mathcal{M})$,
\begin{equation}
\Pr[\mathcal{M}(D_1) \in S] \leq e^{\epsilon} \cdot \Pr[\mathcal{M}(D_2) \in S] + \delta.
\label{cdp}
\end{equation}
Here, $\epsilon > 0$ is the privacy budget (smaller values yield stronger privacy), and $\delta \in [0,1)$ is a small additive slack term representing the probability of the privacy guarantee failing to hold.

For a function $f$ with L2 sensitivity $\Delta f = \max_{D_1 \sim D_2} \|f(D_1) - f(D_2)\|_2$, adding noise $\xi \sim \mathcal{N}(0, \sigma^2 \mathbf{I})$ with $\sigma = \frac{\Delta f}{\epsilon_0}\sqrt{2\ln(1.25/\delta_0)}$ yields an $(\epsilon_0,\delta_0)$-differentially private mechanism~\cite{abadi2016deep,Dwork2013}. DP-SimAgg instantiates this mechanism at each round using the empirically estimated sensitivity $\Delta f$ from the neighboring dataset simulation described above.

In the proposed framework, GDP is implemented directly at the aggregation server. After receiving the local model tensors from collaborators, the server first applies L2 clipping to ensure that all updates lie within a bounded norm. The clipped updates are then aggregated using the similarity-based weighting scheme described previously. To estimate the sensitivity $\Delta f$ of the aggregation function, a simulated neighboring dataset is constructed by replacing one collaborator update with a adversarial tensor 
bounded by the clipping threshold $C$. The L2 difference between the aggregation outputs for the original and simulated datasets provides an empirical estimate of $\Delta f$.

Once the sensitivity $\Delta f$ has been determined, Gaussian noise is sampled and added to the aggregated tensor according to the calibrated differential privacy mechanism. This centralized privacy protection minimizes cumulative noise during training while mitigating risks associated with model inversion, membership inference, and related inference-based attacks~\cite{DBLP:journals/corr/ShokriSS16,https://doi.org/10.48550/arxiv.2003.14053}.

\subsection{Privacy Budget and Noise-Accuracy Trade-off}

The level of privacy protection is controlled by the per-round privacy budget $\epsilon_0 > 0$, which governs the trade-off between privacy and model accuracy. Smaller values of $\epsilon_0$ correspond to stronger per-round privacy guarantees, but require injecting larger noise, which may reduce model performance. The parameter $\delta_0 \in (0,1)$ denotes the per-round probability of a privacy guarantee violation and is chosen to be negligibly small in practice.

The noise magnitude is calibrated to the L2 sensitivity $\Delta f$ of the aggregation function and the per-round parameters $(\epsilon_0, \delta_0)$. By the Gaussian mechanism~\cite{abadi2016deep,Dwork2013}, the noise standard deviation is set to
\[
\sigma = \frac{\Delta f}{\epsilon_0}
\sqrt{2 \ln\!\left(\frac{1.25}{\delta_0}\right)},
\]
which guarantees that injecting noise $\xi \sim \mathcal{N}(0, \sigma^2 \mathbf{I})$ into the aggregated tensor yields an $(\epsilon_0, \delta_0)$-differentially private mechanism for that round.

Algorithms~1--3 implement the complete privacy-preserving aggregation procedure. Algorithm~1 ($\textsc{DP-SimAgg}$) orchestrates the full aggregation pipeline, calling Algorithm~3 ($\textsc{L2Clip}$) to bound collaborator updates and Algorithm~2 ($\textsc{AddGaussianNoise}$) to inject calibrated Gaussian noise. Together, these components ensure that all collaborator tensors are constrained within a predefined L2 norm ball of radius $C$ prior to aggregation, thereby bounding sensitivity, limiting the influence of outlier updates, and stabilizing training.

In our experiments, federated training was conducted for 20 communication rounds. Each round applies one invocation of the Gaussian mechanism with per-round budget $(\epsilon_0, \delta_0)$. By the sequential composition theorem for differential privacy~\cite{Dwork2013}, the total privacy cost after $T$ rounds is at most $(T\epsilon_0,\, T\delta_0)$-differential privacy under basic composition. Tighter bounds can be obtained via advanced composition or the moments accountant~\cite{abadi2016deep}, but we report the basic sequential bound here for transparency. Specifically, when $\epsilon_0 = 1$ per round, the total budget after 20 rounds is $\epsilon_{\text{total}} = 20$; when $\epsilon_0 = 10$, the total budget is $\epsilon_{\text{total}} = 200$ for the monitored weight tensor.
The failure probability parameter $\delta_0$ was fixed at $10^{-5}$ per round (total $\delta_{\text{total}} \leq 20 \times 10^{-5} = 2\times10^{-4}$), while the L2 clipping bound was set to $C = 100$. These parameters provide a practical balance between privacy protection and segmentation performance for the considered medical imaging task.

\begin{algorithm}
\caption{$\textsc{DP-SimAgg}$: Similarity-Weighted Aggregation with Central $(\epsilon,\delta)$-Differential Privacy}
\begin{algorithmic}[1]
\Require $local\_tensors$, $db\_iterator$, $tensor\_name$, $fl\_round$, $collaborators\_chosen$, $collaborator\_times$, per-round $\epsilon_0 > 0$, per-round $\delta_0 \in (0,1)$, clipping bound $C > 0$
\Ensure Aggregated tensor $\tilde{pm}$ satisfying $(\epsilon_0,\delta_0)$-DP for this round

\If{$tensor\_name$ corresponds to weight or bias}
    \State $\epsilon_{safe} \gets 10^{-8}$
    \State Extract tensor values: $T \gets [t.\text{tensor} \text{ for } t \in local\_tensors]$
    \State Extract volume weights: $W \gets [t.\text{weight} \text{ for } t \in local\_tensors]$
    
    \Comment{Step 1: Apply L2 Clipping (Algorithm~3) with bound $C$}
    \For{each $i$}
        \State $T[i] \gets \textsc{L2Clip}(T[i],\; C)$
    \EndFor
    
    \Comment{Step 2: Compute Similarity-Weighted Aggregate}
    \State $\bar{T} \gets \frac{1}{|T|}\sum_{i} T[i]$ \Comment{unweighted mean for distance reference}
    
    \State Compute L2 distances: $d_i \gets \|\bar{T} - T[i]\|_2$ for each $i$
    
    \State Compute total distance: $total\_dist \gets \sum_{i} d_i$
    
    \For{each collaborator $i$}
        \State Similarity factor: $s_i \gets \frac{total\_dist}{\epsilon_{safe} + d_i}$
    \EndFor
    
    \State Combined weights: $\hat{w}_i \gets W[i] + s_i$ for each $i$
    
    \State Normalize: $W_{norm}[i] \gets \frac{\hat{w}_i}{\epsilon_{safe} + \sum_{j} \hat{w}_j}$
    
    \State Weighted aggregate: $pm \gets \sum_{i} W_{norm}[i] \cdot T[i]$
    
    \Comment{Step 3: Estimate L2 Sensitivity via Neighboring Dataset}
    \State Copy $T' \gets T$; substitute $T'[0] \gets \mathbf{1} \cdot C$ \Comment{worst-case neighbor}
    
    \State Compute neighboring aggregate: $pm' \gets \sum_{i} W_{norm}[i] \cdot T'[i]$
    
    \State Estimate sensitivity: $\Delta f \gets \|pm - pm'\|_2$
    
    \Comment{Step 4: Inject Gaussian Noise for Differential Privacy}
    \State $\tilde{pm} \gets \textsc{AddGaussianNoise}(pm,\; \Delta f,\; \epsilon_0,\; \delta_0)$
    
    \State \Return $\tilde{pm}$
\Else
    \State Compute data-volume-weighted average: $pm \gets \frac{\sum_{i} W[i] \cdot T[i]}{\sum_{i} W[i]}$
    \State Set $\Delta f \gets C$ \Comment{conservative sensitivity bound for non-weight tensors}
    \State Apply noise: $\tilde{pm} \gets \textsc{AddGaussianNoise}(pm,\; \Delta f,\; \epsilon_0,\; \delta_0)$
    \State \Return $\tilde{pm}$
\EndIf
\end{algorithmic}
\end{algorithm}

\begin{algorithm}
\caption{$\textsc{AddGaussianNoise}$: Gaussian Noise Injection for $(\epsilon,\delta)$-DP}
\begin{algorithmic}[1]
\Require Aggregated tensor $pm$, L2 sensitivity $\Delta f$, privacy budget $\epsilon > 0$, failure probability $\delta \in (0,1)$
\Ensure Noisy aggregated tensor $\tilde{pm}$ satisfying $(\epsilon,\delta)$-DP

\Comment{Step 1: Compute Noise Standard Deviation (Gaussian Mechanism)}
\State
\[
\sigma = \frac{\Delta f}{\epsilon} \cdot \sqrt{2 \ln(1.25 / \delta)}
\]

\Comment{Step 2: Sample Gaussian Noise}
\State
\[
\xi \sim \mathcal{N}\!\left(0,\; \sigma^2 \mathbf{I}\right)
\]

\Comment{Step 3: Add Noise to Aggregated Tensor}
\State
\[
\tilde{pm} = pm + \xi
\]

\State \Return $\tilde{pm}$

\end{algorithmic}
\end{algorithm}

\begin{algorithm}
\caption{$\textsc{L2Clip}$: L2 Norm Clipping for Differential Privacy}
\begin{algorithmic}[1]
\Require Input tensor $v$, clipping bound $C > 0$
\Ensure Clipped tensor $\tilde{v}$ with $\|\tilde{v}\|_2 \leq C$

\Comment{Step 1: Convert input to array}
\State $v \gets \text{asarray}(v)$

\Comment{Step 2: Handle Scalar Case}
\If{$v$ is a scalar}
    \State \Return $v$ \Comment{scalars need no clipping}
\EndIf

\Comment{Step 3: Flatten for Norm Computation}
\State $v_{flat} \gets \text{flatten}(v)$

\Comment{Step 4: Compute L2 Norm}
\State $\|v_{flat}\|_2 \gets \sqrt{\sum_j v_{flat}[j]^2}$

\Comment{Step 5: Clip if Norm Exceeds Bound}
\If{$\|v_{flat}\|_2 > C$}
    \State $v_{flat} \gets \dfrac{C}{\|v_{flat}\|_2} \cdot v_{flat}$
    \State $\tilde{v} \gets \text{reshape}(v_{flat},\; \text{shape}(v))$
    \State \Return $\tilde{v}$
\Else
    \State \Return $v$
\EndIf

\end{algorithmic}
\end{algorithm}


\section{Results}
\label{sec:results}

We assess the performance of a 3D U-Net model trained on the FeTS~2022 dataset in a federated setting using our proposed DP-SimAgg methodology. For comparison, we trained for 20 federation rounds using the non-private SimAgg baseline and DP-SimAgg with per-round budgets $\epsilon_0 \in \{1, 10\}$.
The total simulation training time per FL run is approximately 16--18 hours for both DP-SimAgg and SimAgg. Notably, incorporating differential privacy through DP-SimAgg does not significantly increase the computational overhead, confirming that the noise injection step is computationally inexpensive relative to the local model training cost.

Figure~\ref{fig_dp_simagg} illustrates the performance of DP-SimAgg with per-round privacy budgets $\epsilon_0 \in \{1.0, 10\}$ compared to the non-private SimAgg baseline. All methods achieve stable performance within approximately 10 federated rounds. 
Table~\ref{internal_validation_table} details the computational and performance metrics after 20 rounds of training. The segmentation labels follow the BraTS convention: label~1 denotes the necrotic tumor core (NCR), label~2 the peritumoral edema (ED), and label~4 the GD-enhancing tumor (ET); label~0 is background (non-tumor). For reporting, the composite regions are: enhancing tumor (ET, label~4), tumor core (TC = NCR $\cup$ ET, labels~1$+$4), and whole tumor (WT = NCR $\cup$ ED $\cup$ ET, labels~1$+$2$+$4).

DP-SimAgg with $\epsilon_0=1.0$ demonstrates competitive segmentation accuracy, achieving 
Dice scores of 0.9967 for background (label 0), 0.6357 for enhancing tumor (ET, label 4), 0.5305 for tumor core (TC), and 0.5274 for whole tumor (WT).
Increasing the per-round privacy budget to $\epsilon_0=10$ substantially improves the Dice scores for all lesion regions, closely matching the baseline SimAgg performance. Computationally, DP-SimAgg demonstrates efficiency comparable to SimAgg, with minor variations in memory usage (${\approx}305$~GB), training duration (16--18 hours), and energy consumption (${\approx}3400$~Wh), validating its practicality for privacy-preserving federated learning applications.

The experimental evaluation highlights the practical applicability of our DP-SimAgg methodology for privacy-preserving federated learning in medical image segmentation tasks. By incorporating differential privacy and leveraging the SimAgg algorithm, we achieve a balance between privacy preservation and model accuracy, enabling the development of robust and generalizable segmentation models while ensuring the protection of sensitive patient data.
\begin{figure*}
    \centering
    \includegraphics[width=\textwidth, angle=0]{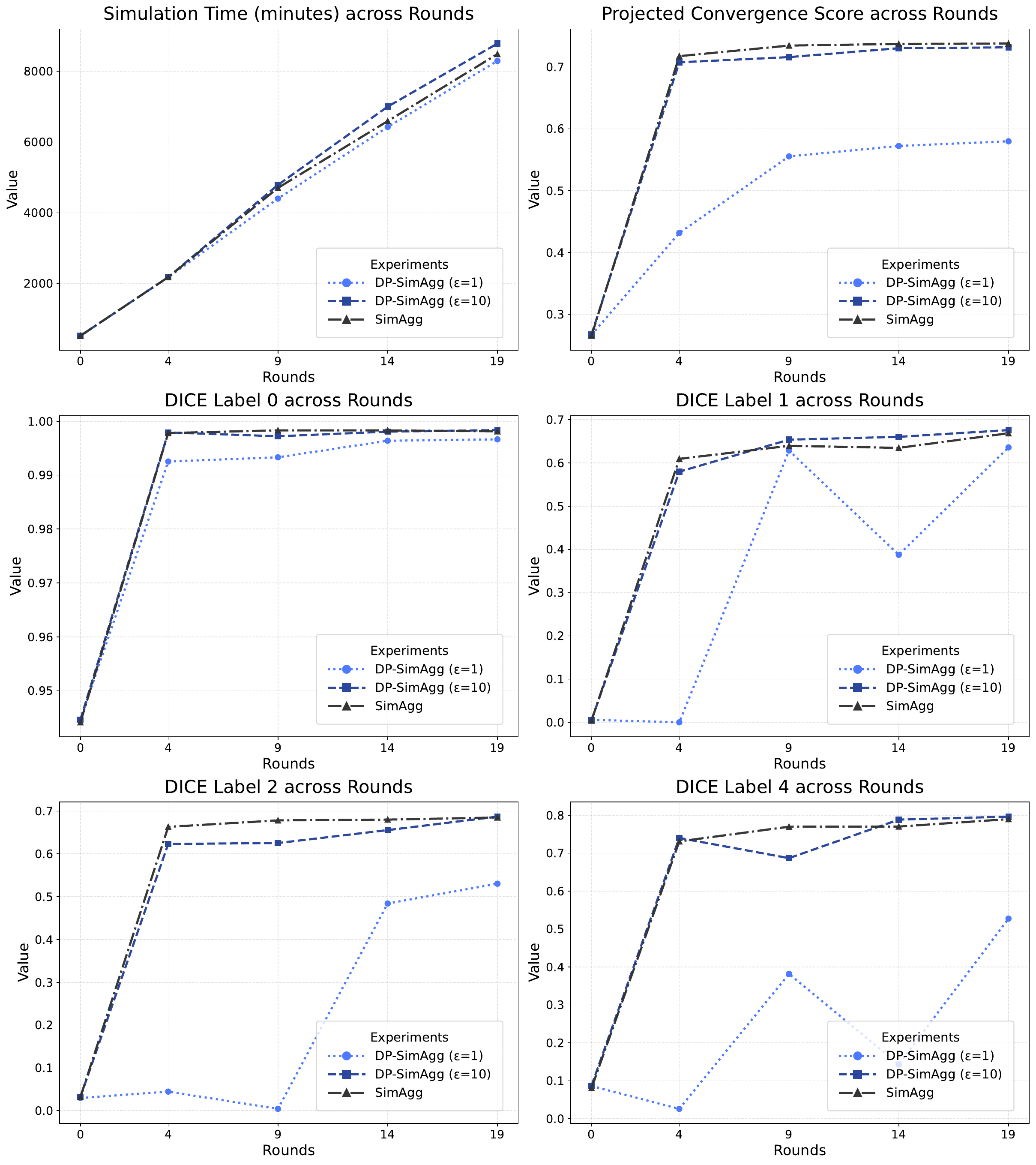}
    \caption{DP-SimAgg with per-round budgets $\epsilon_0 \in \{1.0, 10\}$ and non-private SimAgg across multiple segmentation metrics over 20 federated training rounds.}
    \label{fig_dp_simagg}
\end{figure*}

\begin{table*}[t]
\caption{Performance comparison of DP-SimAgg and SimAgg on the internal validation dataset after 20 training rounds. $\epsilon$ values are \emph{per-round} budgets; cumulative budgets under basic sequential composition are $\epsilon_{\text{total}}=20$ and $\epsilon_{\text{total}}=200$, respectively. DICE Label~0: background; Label~1: necrotic core (NCR); Label~2: peritumoral edema (ED) / tumor core (TC); Label~4: GD-enhancing tumor (ET) / whole tumor (WT) composite.}
\label{internal_validation_table}
\vskip 0.15in
\begin{center}
\begin{small}
\begin{sc}
\begin{tabular}{lccc}
\toprule
Metrics & DP-SimAgg ($\epsilon_0{=}1$, $\epsilon_{\text{tot}}{=}20$) & DP-SimAgg ($\epsilon_0{=}10$, $\epsilon_{\text{tot}}{=}200$) & SimAgg (no DP) \\
\midrule
Job Wall-clock time       & 16:08:53 & 17:55:35 & 18:01:56 \\
Memory Utilized (GB)      & 305.43   & 305.88   & 305.32 \\
Energy (Wh)               & 3393.88  & 3780.68  & 3396.47 \\
Simulation Time (min)     & 8291.95  & 8781.74  & 8484.11 \\
Projected Conv. Score     & \textbf{0.5799} & \textbf{0.7318} & \textbf{0.7380} \\
DICE Label 0              & 0.9967   & 0.9984   & 0.9981 \\
DICE Label 1              & 0.6357   & 0.6757   & 0.6685 \\
DICE Label 2              & 0.5305   & 0.6872   & 0.6851 \\
DICE Label 4              & 0.5274   & 0.7962   & 0.7896 \\
\bottomrule
\end{tabular}
\end{sc}
\end{small}
\end{center}
\vskip -0.1in
\end{table*}

\section{Discussion}
\label{sec:discussion}

\subsection*{Privacy--Utility Trade-off}
 
The results from Figure~\ref{fig_dp_simagg} and Table~\ref{internal_validation_table} confirm that DP-SimAgg maintains competitive segmentation performance while 
providing per-round $(\epsilon_0,\delta_0)$-DP guarantees under the chosen sensitivity bound
All methods converge within approximately 10 federated rounds, demonstrating the practical efficiency of the framework. Under a strict per-round budget ($\epsilon_0=1.0$), noise injection reduces Dice scores by roughly 10--15 percentage points relative to non-private SimAgg, particularly for the tumor core and whole tumor regions. Relaxing the budget to $\epsilon_0=10$ closes most of this gap, with Dice scores closely matching SimAgg across all lesion classes. Computational overhead is negligible: wall-clock time, memory, and energy consumption differ by less than 5\% between DP-SimAgg and SimAgg regardless of privacy setting.
 
The Dice reductions observed under $\epsilon_0=1$ warrant careful clinical interpretation. Performance differences of this magnitude 
fall within the range of inter-rater variability reported in multi-institutional glioblastoma segmentation studies~\cite{bakas_advancing_2017}, suggesting that the framework may remain practically useful for research-oriented or decision-support settings
even at strict privacy budgets, particularly in screening or pre-operative planning contexts. That said, deployment in fully automated clinical pipelines would require prospective validation on institution-specific data.
 
\subsection*{Central vs.\ Local Differential Privacy}
 
DP-SimAgg employs central (server-side) DP, which differs fundamentally from local DP in both its trust model and noise characteristics. In local DP, each collaborator independently perturbs its update before transmission, requiring noise calibrated to the per-collaborator sensitivity. This results in substantially higher total noise at the server after aggregation. Central DP, by contrast, applies a single calibrated noise injection at the aggregator after receiving all local tensors, with noise scaled to the \emph{global} L2 sensitivity $\Delta f$ of the aggregation function. For a fixed $(\epsilon_0,\delta_0)$ target, this yields a lower noise magnitude and better model utility~\cite{kaissis2020}. The near-baseline performance at $\epsilon_0=10$ directly reflects this efficiency. The trade-off is the requirement for a trusted aggregator is a standard assumption in central DP but one that may not hold in all federated deployments.
 
Gaussian noise injection also reduces risk or offers principled protection against information leakage,
 including model inversion and membership inference~\cite{DBLP:journals/corr/ShokriSS16,https://doi.org/10.48550/arxiv.2003.14053}. By obscuring individual collaborator contributions in the aggregated tensor, DP-SimAgg reduces the risk of patient re-identification, supporting compliance with privacy regulations such as GDPR and HIPAA.
 
\subsection*{Limitations and Future Directions}
 
Two methodological limitations warrant particular attention. First, DP-SimAgg requires a trusted central aggregator. This is standard in the central DP model but may be unacceptable in fully decentralized or adversarial-aggregator settings. Future work could combine DP-SimAgg with secure aggregation protocols~\cite{mugunthan_peraire_bueno_kagal_2020} or investigate hybrid schemes that apply local DP at the collaborator level with a lighter central post-processing step. Second, the sensitivity $\Delta f$ is estimated empirically via a simulated neighboring dataset rather than derived analytically from the specific SimAgg weight structure. As shown in the Supplementary Material (Section~S1, Proposition~S1), a conservative analytical bound of $4C$ holds 
for aggregator, but is empirically loose. 
 
Beyond these core issues, additional practical limitations include: (i) evaluation is restricted to a single brain tumour segmentation task — broader validation across imaging modalities and pathologies is needed; (ii) the fixed 20\% collaborator sampling rate may be suboptimal, so adaptive selection based on update divergence or data quality could accelerate convergence; and (iii) scalability to larger collaborator pools or higher-dimensional models has not been assessed. Tighter privacy accounting via the moments accountant~\cite{abadi2016deep} or R\'{e}nyi DP composition could also reduce the reported cumulative budget without changing the per-round mechanism.

While prior work has extensively demonstrated the superiority of SimAgg over standard federated optimization methods such as FedAvg in non-IID medical imaging settings, the primary focus of this study is to investigate its behavior under differential privacy constraints. We acknowledge that direct comparisons with differentially private baselines (e.g., DP-FedAvg) and a detailed ablation study isolating the contributions of similarity-weighted aggregation and privacy mechanisms would further strengthen the empirical analysis. Additionally, although standard techniques such as L2 norm clipping and Gaussian noise injection are employed, a more rigorous treatment of privacy accounting—such as tighter bounds using R\'{e}nyi DP remains an important direction for future work. The observed degradation in segmentation performance with stricter privacy budgets is consistent with the expected privacy–utility trade-off, with varying sensitivity across tumor subregions suggesting the influence of structural complexity and class imbalance. Despite these limitations, this work provides, to our knowledge, the first empirical study of similarity-weighted aggregation within a differentially private federated learning framework for brain tumor segmentation, demonstrating that the proposed approach maintains competitive performance even under strong privacy constraints and highlighting its potential for privacy-preserving medical AI applications.

\section{Conclusion}
We proposed DP-SimAgg, a federated learning framework that integrates similarity-weighted aggregation with server-side Gaussian differential privacy for brain lesion segmentation. The method bounds collaborator influence via L2 clipping, adapts aggregation weights to non-IID data distributions through a similarity-based scheme, and injects calibrated Gaussian noise at the central aggregator to provide formal $(\epsilon_0,\delta_0)$-DP guarantees per communication round. Evaluated on the FeTS~2022 dataset with 1251 multi-institutional MRI scans, DP-SimAgg achieves competitive Dice scores even under a strict per-round budget ($\epsilon_0=1$, $\epsilon_{\text{total}}=20$), and approaches non-private SimAgg performance at $\epsilon_0=10$, with minor additional overhead computational cost.
 
Key open directions include: deriving a tighter analytical sensitivity bound to replace the current empirical estimate; extending the framework to untrusted-aggregator settings via secure aggregation; applying tighter composition accountants to reduce the reported cumulative privacy cost; and validating across broader clinical imaging tasks and collaborator scales.

\section*{Author Biographies}

\textbf{Muhammad Irfan Khan, M.Sc.}, is a Senior Researcher in Turku University of Applied Sciences, Finland. His research focuses on bayesian machine learning and federated learning. 

\textbf{Eero Lehtonen, D.Sc. (Tech.)}, (Electronics), is a Principal Lecturer and the leader of the INSIGHT research group at Turku University of Applied Sciences, Finland. His research focuses on computer vision, applied machine learning, and AI‑based methods for medical imaging.

\textbf{Joni Obradovic, B.Eng}, Health Tech (ICT), has experience in research and development of AI solutions using privacy-preserving federated learning and agentic workflows. He has expertise in leveraging supercomputing resources for AI solutions and works as a developer in the proposed project, with an emphasis on federated and remotely orchestrated components of the software.

\textbf{Elina Kontio, PhD} is currently the Head of Education and Research in Data Engineering and AI Technologies at Turku University of Applied Sciences, Finland. She has over 20 years of experience in higher education, research leadership, and curriculum development in health technology and ICT. Her research interests include privacy-preserving artificial intelligence, federated learning, medical image analysis, and data-driven decision-making in healthcare. 

\textbf{Esa Alhoniemi, PhD} is an Adjunct Professor and Senior Data Scientist with over three decades of experience in academia and industry. He worked at Helsinki University of Technology and the University of Turku, where his research focused on data mining and algorithmics. Since 2008, he has held senior industry roles in bioinformatics, predictive analytics, data automation, and data science. He is currently a Senior Advisor at Turku University of Applied Sciences, working on data science applications in business.

\textbf{Suleiman A. Khan, PhD} has expertise in artificial intelligence, machine learning, and data-driven systems. He is currently affiliated with Turku University of Applied Sciences, Finland, where he works in ICT, data engineering, and AI technologies. His research interests include applied machine learning, federated learning, medical image analysis, privacy-preserving AI, and intelligent software systems. He has contributed to academic and applied research projects involving AI-based healthcare and data science applications.

\textbf{Mojtaba Jafaritadi, PhD} received the doctoral degree in Medical Physics and Engineering from the University of Turku, Finland. He is an Adjunct Professor of Health AI at the University of Turku and previously worked as a Postdoctoral Research Fellow at Stanford University and Principal Lecturer at Turku University of Applied Sciences. His research focuses on generative AI, privacy-preserving machine learning, biomedical signal and image processing, medical imaging, and multimodal learning systems.

\section*{Acknowledgements}
This work was partially supported by the Business Finland under grants 10882/31/2022 and 1337/31/2024 and by the EU Horizon project Phase IV AI (grant number 101095384) . We also thank CSC-Puhti super-computer for the computational resources. All sections were improved by using AI including chatGPT and Claude.

\bibliographystyle{IEEEtran}
\bibliography{example_paper}

\begin{thebibliography}{10}
\providecommand{\url}[1]{#1}
\csname url@samestyle\endcsname
\providecommand{\newblock}{\relax}
\providecommand{\bibinfo}[2]{#2}
\providecommand{\BIBentrySTDinterwordspacing}{\spaceskip=0pt\relax}
\providecommand{\BIBentryALTinterwordstretchfactor}{4}
\providecommand{\BIBentryALTinterwordspacing}{\spaceskip=\fontdimen2\font plus
\BIBentryALTinterwordstretchfactor\fontdimen3\font minus
  \fontdimen4\font\relax}
\providecommand{\BIBforeignlanguage}[2]{{%
\expandafter\ifx\csname l@#1\endcsname\relax
\typeout{** WARNING: IEEEtran.bst: No hyphenation pattern has been}%
\typeout{** loaded for the language `#1'. Using the pattern for}%
\typeout{** the default language instead.}%
\else
\language=\csname l@#1\endcsname
\fi
#2}}
\providecommand{\BIBdecl}{\relax}
\BIBdecl

\bibitem{varoquaux2022machine}
G.~Varoquaux and V.~Cheplygina, ``Machine learning for medical imaging:
  Methodological failures and recommendations for the future,'' \emph{NPJ
  Digital Medicine}, vol.~5, no.~1, p.~48, 2022.

\bibitem{sheller2020federated}
M.~J. Sheller, B.~Edwards, G.~A. Reina, J.~Martin, S.~Pati, A.~Kotrotsou,
  M.~Milchenko, W.~Xu, D.~Marcus, R.~R. Colen \emph{et~al.}, ``Federated
  learning in medicine: Facilitating multi-institutional collaborations without
  sharing patient data,'' \emph{Scientific Reports}, vol.~10, no.~1, p. 12598,
  2020.

\bibitem{bakas_advancing_2017}
S.~Bakas, H.~Akbari, A.~Sotiras, M.~Bilello, M.~Rozycki, J.~S. Kirby, J.~B.
  Freymann, K.~Farahani, and C.~Davatzikos, ``Advancing {The Cancer Genome
  Atlas} glioma {MRI} collections with expert segmentation labels and radiomic
  features,'' \emph{Scientific Data}, vol.~4, no.~1, p. 170117, 2017.

\bibitem{may2010hipaa2}
M.~May, ``Hipaa legislation means more delicate handling of data,''
  \emph{Nature Medicine}, vol.~16, no.~3, pp. 250--251, 2010.

\bibitem{tinja}
T.~Pitk{\"a}m{\"a}ki, T.~Pahikkala, I.~M. Perez, P.~Movahedi, V.~Nieminen,
  T.~Southerington, J.~Vaiste, M.~Jafaritadi, M.~I. Khan, E.~Kontio,
  P.~Ranttila, J.~Pajula, H.~P{\"o}l{\"o}nen, A.~Degerli, J.~Plomp, and
  A.~Airola, ``Finnish perspective on using synthetic health data to protect
  privacy: The privasa project,'' \emph{Applied Computing and Intelligence},
  vol.~4, no.~2, pp. 138--163, 2024.

\bibitem{ng2021federated}
D.~Ng, X.~Lan, M.~M.-S. Yao, W.~P. Chan, and M.~Feng, ``Federated learning: A
  collaborative effort to achieve better medical imaging models for individual
  sites that have small labelled datasets,'' \emph{Quantitative Imaging in
  Medicine and Surgery}, vol.~11, no.~2, p. 852, 2021.

\bibitem{mcmahan2017communication}
H.~B. McMahan, E.~Moore, D.~Ramage, S.~Hampson, and B.~A. y~Arcas,
  ``Communication-efficient learning of deep networks from decentralized
  data,'' in \emph{Proceedings of the 20th International Conference on
  Artificial Intelligence and Statistics}, 2017, pp. 1273--1282.

\bibitem{wang_yurochkin_papailiopoulos_khazaeni_2020}
H.~Wang, M.~Yurochkin, Y.~Sun, D.~Papailiopoulos, and Y.~Khazaeni, ``Federated
  learning with matched averaging,'' \emph{arXiv preprint arXiv:2002.06440},
  2020.

\bibitem{khan_regsimagg23}
M.~I. Khan, M.~A. Azeem, E.~Alhoniemi, E.~Kontio, S.~A. Khan, and
  M.~Jafaritadi, ``Regularized weight aggregation in networked federated
  learning for glioblastoma segmentation,'' \emph{arXiv preprint
  arXiv:2301.12617}, 2023.

\bibitem{zhao2018federated}
Y.~Zhao, M.~Li, L.~Lai, N.~Suda, D.~Civin, and V.~Chandra, ``Federated learning
  with non-iid data,'' \emph{arXiv preprint arXiv:1806.00582}, 2018.

\bibitem{Kairouz2019}
P.~Kairouz, H.~B. McMahan, B.~Avent, A.~Bellet, M.~Bennis, A.~N. Bhagoji,
  K.~Bonawitz, Z.~Charles, G.~Cormode, R.~Cummings \emph{et~al.}, ``Advances
  and open problems in federated learning,'' \emph{Foundations and Trends in
  Machine Learning}, vol.~14, no. 1--2, pp. 1--210, 2021.

\bibitem{karimireddy2020scaffold}
S.~P. Karimireddy, S.~Kale, M.~Mohri, S.~Reddi, S.~Stich, and A.~T. Suresh,
  ``{SCAFFOLD}: Stochastic controlled averaging for federated learning,'' in
  \emph{ICML}, 2020, pp. 5132--5143.

\bibitem{https://doi.org/10.48550/arxiv.2003.14053}
J.~Geiping, H.~Bauermeister, H.~Dr{\"o}ge, and M.~Moeller, ``Inverting
  gradients -- how easy is it to break privacy in federated learning?''
  \emph{arXiv preprint arXiv:2003.14053}, 2020.

\bibitem{DBLP:journals/corr/ShokriSS16}
R.~Shokri, M.~Stronati, and V.~Shmatikov, ``Membership inference attacks
  against machine learning models,'' \emph{CoRR}, vol. abs/1610.05820, 2016.

\bibitem{truong_privacy_2021}
N.~Truong, K.~Sun, S.~Wang, F.~Guitton, and Y.~Guo, ``Privacy preservation in
  federated learning: An insightful survey from the {GDPR} perspective,''
  \emph{arXiv preprint arXiv:2011.05411}, 2021.

\bibitem{rodriguez2023survey}
N.~Rodr{\'i}guez-Barroso, D.~Jim{\'e}nez-L{\'o}pez, M.~V. Luz{\'o}n,
  F.~Herrera, and E.~Mart{\'i}nez-C{\'a}mara, ``Survey on federated learning
  threats: Concepts, taxonomy on attacks and defences, experimental study and
  challenges,'' \emph{Information Fusion}, vol.~90, pp. 148--173, 2023.

\bibitem{khan2021adaptive}
M.~I. Khan, M.~Jafaritadi, E.~Alhoniemi, E.~Kontio, and S.~A. Khan, ``Adaptive
  weight aggregation in federated learning for brain tumor segmentation,'' in
  \emph{International MICCAI Brainlesion Workshop}.\hskip 1em plus 0.5em minus
  0.4em\relax Springer, 2021, pp. 455--469.

\bibitem{zenk2025towards}
M.~Zenk, U.~Baid, S.~Pati, A.~Linardos, B.~Edwards, M.~Sheller, P.~Foley,
  A.~Aristizabal, D.~Zimmerer, A.~Gruzdev \emph{et~al.}, ``Towards fair
  decentralized benchmarking of healthcare {AI} algorithms with the federated
  tumor segmentation ({FeTS}) challenge,'' \emph{Nature Communications},
  vol.~16, no.~1, p. 6274, 2025.

\bibitem{Linardos_2025}
A.~Linardos, S.~Pati, U.~Baid, B.~Edwards, P.~Foley, K.~Ta, V.~Chung,
  M.~Sheller, M.~I. Khan, M.~Jafaritadi, E.~Kontio, S.~Khan, L.~Machler,
  I.~Ezhov, S.~Shit, J.~C. Paetzold, G.~Grimberg, M.~A. Nickel, D.~Naccache,
  V.~Siomos, J.~Passerat-Palmbach, G.~Tarroni, D.~Kim, L.~L. Klausmann,
  P.~Shah, B.~Menze, D.~Makris, and S.~Bakas, ``The {MICCAI} federated tumor
  segmentation ({FeTS}) challenge 2024: Efficient and robust aggregation
  methods,'' \emph{Machine Learning for Biomedical Imaging}, vol.~3, pp.
  757--774, 2025.

\bibitem{reina_openfl_2022}
G.~A. Reina, A.~Gruzdev, P.~Foley, O.~Perepelkina, M.~Sharma, I.~Davidyuk,
  I.~Trushkin, M.~Radionov, A.~Mokrov, D.~Agapov, J.~Martin, B.~Edwards, M.~J.
  Sheller, S.~Pati, P.~N. Moorthy, S.-h. Wang, P.~Shah, and S.~Bakas, ``Openfl:
  An open-source framework for federated learning,'' \emph{Physics in Medicine
  and Biology}, 2022.

\bibitem{li2020federated}
T.~Li, A.~K. Sahu, M.~Zaheer, M.~Sanjabi, A.~Talwalkar, and V.~Smith,
  ``Federated optimization in heterogeneous networks,'' in \emph{MLSys}, 2020.

\bibitem{wang2020tackling}
J.~Wang, Q.~Liu, H.~Liang, G.~Joshi, and H.~V. Poor, ``Tackling the objective
  inconsistency problem in heterogeneous federated optimization,'' in
  \emph{Advances in Neural Information Processing Systems}, vol.~33, 2020, pp.
  7611--7623.

\bibitem{blanchard2017machine}
P.~Blanchard, E.~M. El~Mhamdi, R.~Guerraoui, and J.~Stainer, ``Machine learning
  with adversaries: Byzantine tolerant gradient descent,'' in \emph{Advances in
  Neural Information Processing Systems}, vol.~30, 2017.

\bibitem{yin2018byzantine}
D.~Yin, Y.~Chen, R.~Kannan, and P.~Bartlett, ``Byzantine-robust distributed
  learning: Towards optimal statistical rates,'' in \emph{Proceedings of the
  35th International Conference on Machine Learning}, ser. Proceedings of
  Machine Learning Research, vol.~80, 2018, pp. 5650--5659.

\bibitem{khanadaptive}
M.~I. Khan, M.~Jafaritadi, E.~Alhoniemi, E.~Kontio, and S.~A. Khan, ``Adaptive
  weight aggregation in federated learning for brain tumor segmentation,'' in
  \emph{Brainlesion: Glioma, Multiple Sclerosis, Stroke and Traumatic Brain
  Injuries}, A.~Crimi and S.~Bakas, Eds.\hskip 1em plus 0.5em minus 0.4em\relax
  Cham: Springer International Publishing, 2022, pp. 455--469.

\bibitem{10.1007/11787006_1}
C.~Dwork, ``Differential privacy,'' in \emph{Automata, Languages and
  Programming}, 2006, pp. 1--12.

\bibitem{Dwork2013}
C.~Dwork and A.~Roth, ``The algorithmic foundations of differential privacy,''
  \emph{Foundations and Trends in Theoretical Computer Science}, vol.~9, no.
  3--4, pp. 211--407, 2014.

\bibitem{Dwork2011}
C.~Dwork, ``A firm foundation for private data analysis,'' \emph{Communications
  of the ACM}, vol.~54, no.~1, pp. 86--95, 2011.

\bibitem{abadi2016deep}
M.~Abadi, A.~Chu, I.~Goodfellow, H.~B. McMahan, I.~Mironov, K.~Talwar, and
  L.~Zhang, ``Deep learning with differential privacy,'' in \emph{Proceedings
  of the 2016 {ACM} {SIGSAC} Conference on Computer and Communications
  Security}, 2016, pp. 308--318.

\bibitem{geyer2017differentially}
R.~C. Geyer, T.~Klein, and M.~Nabi, ``Differentially private federated
  learning: A client level perspective,'' in \emph{NIPS Workshop on Private
  Multi-Party Machine Learning}, 2017.

\bibitem{DBLP:journals/corr/abs-1905-02383}
J.~Dong, A.~Roth, and W.~J. Su, ``Gaussian differential privacy,'' \emph{CoRR},
  vol. abs/1905.02383, 2019.

\bibitem{mironov2017renyi}
I.~Mironov, ``R{\'e}nyi differential privacy,'' in \emph{2017 IEEE 30th
  Computer Security Foundations Symposium (CSF)}.\hskip 1em plus 0.5em minus
  0.4em\relax IEEE, 2017, pp. 263--275.

\bibitem{khan2024electioncollaboratorsreinforcementlearning}
M.~I. Khan, E.~Kontio, S.~A. Khan, and M.~Jafaritadi, ``Election of
  collaborators via reinforcement learning for federated brain tumor
  segmentation,'' \emph{arXiv preprint arXiv:2412.20253}, 2024.

\bibitem{khan2024recommenderenginedrivenclient}
------, ``Recommender engine driven client selection in federated brain tumor
  segmentation,'' \emph{arXiv preprint arXiv:2412.20250}, 2024.

\bibitem{kaissis2020}
G.~A. Kaissis, M.~R. Makowski, D.~R{\"u}ckert, and R.~F. Braren, ``Secure,
  privacy-preserving and federated machine learning in medical imaging,''
  \emph{Nature Machine Intelligence}, vol.~2, no.~6, pp. 305--311, 2020.

\bibitem{bernstein2018signsgd}
J.~Bernstein, Y.-X. Wang, K.~Azizzadenesheli, and A.~Anandkumar, ``signsgd:
  Compressed optimisation for non-convex problems,'' in \emph{International
  conference on machine learning}.\hskip 1em plus 0.5em minus 0.4em\relax PMLR,
  2018, pp. 560--569.

\bibitem{jin2025noisy}
R.~Jin and H.~Dai, ``Noisy signsgd is more differentially private than you
  (might) think,'' in \emph{Forty-second International Conference on Machine
  Learning}, 2025.

\bibitem{eisenmann2023biomedicalimageanalysiscompetitions}
L.~Maier-Hein \emph{et~al.}, ``Biomedical image analysis competitions: The
  state of current participation practice,'' \emph{arXiv preprint
  arXiv:2212.08568}, 2023.

\bibitem{kleihues1993new}
P.~Kleihues, P.~C. Burger, and B.~W. Scheithauer, ``The new {WHO}
  classification of brain tumours,'' \emph{Brain Pathology}, vol.~3, no.~3, pp.
  255--268, 1993.

\bibitem{hanif2017glioblastoma}
F.~Hanif, K.~Muzaffar, K.~Perveen, S.~M. Malhi, and S.~U. Simjee,
  ``Glioblastoma multiforme: A review of its epidemiology and pathogenesis
  through clinical presentation and treatment,'' \emph{Asian Pacific Journal of
  Cancer Prevention}, vol.~18, no.~1, p.~3, 2017.

\bibitem{bakas_segmentation_2017}
S.~Bakas, H.~Akbari, A.~Sotiras, M.~Bilello, M.~Rozycki, J.~Kirby, J.~Freymann,
  K.~Farahani, and C.~Davatzikos, ``Segmentation labels for the pre-operative
  scans of the {TCGA-GBM} collection,'' 2017, dataset.

\bibitem{bakas_segmentation_2017-1}
------, ``Segmentation labels for the pre-operative scans of the {TCGA-LGG}
  collection,'' 2017, dataset.

\bibitem{ronneberger_u-net_2015}
O.~Ronneberger, P.~Fischer, and T.~Brox, ``U-net: Convolutional networks for
  biomedical image segmentation,'' in \emph{Medical Image Computing and
  Computer-Assisted Intervention ({MICCAI})}, 2015, pp. 234--241.

\bibitem{spyridon_bakas_2022_6362409}
S.~Bakas, S.~Pati, M.~Sheller, A.~Karargyris, P.~Mattson, B.~Edwards, U.~Baid,
  Y.~Chen, R.~T. Shinohara, J.~Martin, B.~Menze, M.~Zenk, K.~Maier-Hein,
  R.~Floca, A.~Reinke, L.~Maier-Hein, F.~Isensee, D.~Zimmerer, and Y.~Chen,
  ``The federated tumor segmentation ({FeTS}) challenge 2022,'' 2022, zenodo /
  challenge description.

\bibitem{mugunthan_peraire_bueno_kagal_2020}
V.~Mugunthan, A.~Peraire-Bueno, and L.~Kagal, ``Privacyfl: A simulator for
  privacy-preserving and secure federated learning,'' \emph{arXiv preprint
  arXiv:2002.08423}, 2020.

\end{thebibliography}

\clearpage
\FloatBarrier
\onecolumn

\section*{Supplementary Material}
This section provides the formal sensitivity analysis for the DP-SimAgg aggregation mechanism and additional architectural diagrams and workflow illustrations.
 
\subsection*{S1.\; Sensitivity Analysis of DP-SimAgg}
 
The following proposition establishes a closed-form \emph{conservative} upper bound on the L2 sensitivity of the clipped similarity-weighted aggregation, valid for any normalized weight function. This bound characterises the worst-case change in the aggregated output when one collaborator's contribution is replaced, and provides the theoretical grounding for the noise calibration performed in Algorithm~1.
 
\begin{proposition}[Conservative sensitivity bound for clipped similarity-weighted aggregation]
\label{prop:sensitivity}
Let
\[
A(T_1,\dots,T_n) = \sum_{i=1}^{n} \alpha_i(T)\,T_i
\]
denote the aggregation output, where each clipped collaborator update satisfies $\|T_i\|_2 \leq C$, and the normalized weights satisfy
\[
\alpha_i(T) \geq 0, \qquad \sum_{i=1}^{n} \alpha_i(T) = 1.
\]
Define two neighboring datasets $T = (T_1,\dots,T_n)$ and $T' = (T_1,\dots,T_k',\dots,T_n)$ differing only in the contribution of collaborator $k$, with $\|T_k'\|_2 \leq C$. Then the $\ell_2$-sensitivity of $A$ satisfies
\[
\Delta f \;=\; \max_{T \sim T'} \|A(T) - A(T')\|_2 \;\leq\; 4C.
\]
\end{proposition}
 
\begin{proof}
Decompose the difference by adding and subtracting $\sum_i \alpha_i(T)\,T_i'$:
\[
A(T) - A(T')
= \underbrace{\sum_{i=1}^{n} \alpha_i(T)(T_i - T_i')}_{\text{(I)}}
+ \underbrace{\sum_{i=1}^{n} \bigl(\alpha_i(T) - \alpha_i(T')\bigr)\,T_i'}_{\text{(II)}}.
\]
Since $T_i = T_i'$ for all $i \neq k$, term~(I) reduces to $\alpha_k(T)(T_k - T_k')$. Applying the triangle inequality to both terms:
\[
\|A(T) - A(T')\|_2
\leq \underbrace{\alpha_k(T)\|T_k - T_k'\|_2}_{\text{bound (I)}}
+ \underbrace{\sum_{i=1}^{n} |\alpha_i(T) - \alpha_i(T')|\,\|T_i'\|_2}_{\text{bound (II)}}.
\]
 
\noindent\textbf{Bounding term (I).} Using $\|T_k\|_2 \leq C$, $\|T_k'\|_2 \leq C$, and $\alpha_k(T) \leq 1$:
\[
\alpha_k(T)\|T_k - T_k'\|_2
\leq \alpha_k(T)\bigl(\|T_k\|_2 + \|T_k'\|_2\bigr)
\leq 2C\,\alpha_k(T)
\leq 2C.
\]
 
\noindent\textbf{Bounding term (II).} Since $\|T_i'\|_2 \leq C$ and $|\alpha_i(T) - \alpha_i(T')| \leq \alpha_i(T) + \alpha_i(T')$:
\[
\sum_{i=1}^{n} |\alpha_i(T) - \alpha_i(T')|\,\|T_i'\|_2 \leq C \sum_{i=1}^{n} (\alpha_i(T) + \alpha_i(T')) = 2C.
\]
 
\noindent Combining both bounds: $\|A(T) - A(T')\|_2 \leq 2C + 2C = 4C$.
\end{proof}
 
\begin{remark}[Conservative nature of the bound and relation to Algorithm~1]
\label{rem:sensitivity}
 
Algorithm~1 does not use the $4C$ bound for noise calibration. Instead, it estimates $\Delta f$ empirically at each round by evaluating $\|A(T) - A(T')\|_2$ for a simulated neighboring dataset $T'$ formed by replacing one collaborator tensor with a 
bounded adversarial tensor. This data-dependent estimate is substantially tighter and yields the noise levels underlying the privacy--utility results in Section~\ref{sec:results}. The analytical $4C$ bound provides a complementary guarantee. 
\end{remark}
\EOD

\end{document}